\documentclass[letterpaper,10pt,conference]{ieeeconf}

\IEEEoverridecommandlockouts
\usepackage{amsmath,amssymb}
\usepackage{mathtools}
\usepackage{bm}
\usepackage{graphicx}
\usepackage{textcomp}
\usepackage{array}
\usepackage{booktabs}
\usepackage[noadjust]{cite}
\usepackage{xcolor}

\newtheorem{theorem}{Theorem}

\newtheorem{proposition}{Proposition}

\newtheorem{definition}{Definition}

\newtheorem{remark}{Remark}
\newtheorem{example}{Example}

\newcommand{\Ne}{\mathbb{N}}
\renewcommand{\Re}{\mathbb{R}}
\newcommand{\Prob}{\mathsf{P}}
\newcommand{\Alg}{\mathcal{A}}
\newcommand{\va}{\mathbf{a}}
\newcommand{\vz}{\mathbf{z}}
\newcommand{\Rg}{\mathcal{R}}
\newcommand{\calS}{\mathcal{S}}
\newcommand{\set}{\mathrm{set}}
\DeclareMathOperator*{\argmax}{arg\,max}

\title{\LARGE\bf PAC Learnability of Scenario Decision-Making Algorithms: Necessary Conditions and Sufficient Conditions}
\author{Guillaume O.~Berger, and Rapha\"el M.~Jungers%
\thanks{This project has received funding from (i) the Belgian National Fund for Scientific Research (F.R.S.--FNRS), (ii) the European Research Council (ERC) [under the European Union's Horizon 2020 research and innovation programme, grant agreement no.~864017, project name: L2C], and (iii) the ARC (French Community of Belgium) [project name: SIDDARTA].}
\thanks{G.~Berger is with ICTEAM institute, UCLouvain, Louvain-la-Neuve, Belgium.
He is an FNRS Postdoctoral Researcher (e-mail: guillaume.berger@uclouvain.be).}
\thanks{R.~Jungers is with ICTEAM institute, UCLouvain, Louvain-la-Neuve, Belgium.
He is an FNRS honorary Research Associate (e-mail: raphael.jungers@uclouvain.be).}}

\begin{document}

\maketitle
\thispagestyle{empty}
\pagestyle{empty}

\begin{abstract}
We investigate the Probably Approximately Correct (PAC) property of scenario decision algorithms, which refers to their ability to produce decisions with an arbitrarily low risk of violating unknown safety constraints, provided a sufficient number of realizations of these constraints are sampled.
While several PAC sufficient conditions for such algorithms exist in the literature---such as the finiteness of the VC dimension of their associated classifiers, or the existence of a compression scheme---it remains unclear whether these conditions are also necessary.
In this work, we demonstrate through counterexamples that these conditions are not necessary in general.
These findings stand in contrast to binary classification learning, where analogous conditions are both sufficient and necessary for a family of classifiers to be PAC.
Furthermore, we extend our analysis to stable scenario decision algorithms, a broad class that includes practical methods like scenario optimization.
Even under this additional assumption, we show that the aforementioned conditions remain unnecessary.
Furthermore, we introduce a novel quantity, called the dVC dimension, which serves as an analogue to the VC dimension for scenario decision algorithms.
We prove that the finiteness of this dimension is a PAC necessary condition for scenario decision algorithms.
This allows to (i) guide algorithm users and designers to recognize algorithms that are not PAC, and (ii) contribute to a comprehensive characterization of PAC scenario decision algorithms.
\end{abstract}

%%%%%%%%%%%%%%%%%%%%%%%%%%%%%%%%%%%%%%%%%%%%%%%%%%%%%%%%%%%%%%%%%%%%%%%%%%%%%%%%
\section{INTRODUCTION}
\label{sec:introduction}

Risk-aware decision making is an important problem in engineering, encompassing many applications of great importance, such as control, energy planning, healthcare, etc.
One key challenge in this problem is that the distribution of the uncertainty is usually unknown, so that one has to rely on observations to make a decision that has a low \emph{risk} (probability of failure).
The approach of \emph{scenario decision making}~\cite{calafiore2006thescenario,campi2008theexact,calafiore2010random,campi2011asamplinganddiscarding,margellos2014ontheroad,esfahani2015performance,grammatico2016ascenario,campi2018waitandjudge,campi2018ageneral,yang2019chanceconstrained,rocchetta2021ascenario,romao2021tight,garatti2022risk,romao2023ontheexact,campi2023compression} provides an effective way to address this problem by using the principle of sample-based methods.
Concretely, this approach consists in sampling $N$ realizations of the uncertainty, and making a decision based on these samples.
The process is called a \emph{scenario decision algorithm}, mapping finite sets of samples to decisions.
Under some conditions on the problem and the algorithm, one can guarantee with high probability that if enough samples are provided to the algorithm, the returned decision has a low risk.
This property is called the \emph{PAC (Probably Approximately Correct)} property.

Several PAC sufficient conditions, i.e., conditions on the problem and the scenario decision algorithm ensuring that the algorithm is PAC, have been studied in the literature; see, e.g.,~\cite{rocchetta2024asurvey} for a survey.
These conditions can be grouped into two categories:%
\footnote{Three categories in~\cite{rocchetta2024asurvey}, but we merged ``compression-based methods'' and ``scenario-based methods'' because of their similarities.}
\emph{complexity-based} conditions and \emph{compression-based} conditions.
Complexity-based conditions, such as those in~\cite{de2004onconstraint,alamo2009randomized,lauer2024marginbased}, provide PAC guarantees for scenario decision problems whose decision space has bounded ``complexity'' such as bounded VC dimension or Rademacher complexity~\cite{rocchetta2024asurvey}.
In this work, we focus on complexity-based PAC conditions that are based on the VC dimension, as in~\cite{alamo2009randomized}.
Compression-based conditions, such as those in~\cite{campi2008theexact,calafiore2010random,margellos2015ontheconnection,campi2023compression}, provide PAC guarantees for scenario decision algorithms whose input can be ``compressed'', meaning that a subset of the samples provides the same decision as the complete set of samples~\cite{rocchetta2024asurvey}.
Note that the results in~\cite{campi2008theexact,calafiore2010random,campi2023compression} require additional assumptions, such as \emph{stability} (see Definition~\ref{def:stable}) and \emph{non-degeneracy} (\cite[Definition~2.7]{calafiore2010random}).

The objective of this work is to progress in the understanding of what makes a scenario decision algorithm PAC or not.
Our starting point is to ask whether the above PAC sufficient conditions are also necessary.
We show with counterexamples that this is not the case (Section~\ref{sec:necessity-pac-sufficient}).
We do this for general scenario decision algorithms, and for \emph{stable} ones (with slightly weaker conclusions); see Table~\ref{tab:results} for a summary.
Designing and analyzing these counterexamples is the first main contribution of this work.
The second main contribution is to provide a novel quantity, similar to the VC dimension, for scenario decision algorithms, and showing that finiteness of this quantity is a PAC necessary condition (Section~\ref{sec:decision-vc-dimension}).
We also show that this PAC necessary condition is not sufficient.

% \begin{figure}
% \centering
% \begin{tikzpicture}
% \node[inner sep=0pt] at (0,0) {\includegraphics[width=0.9\linewidth]{venn}};
% \node[anchor=west,font=\small] at (-1.28,1.22) {\ref{sec:decision-vc-dimension}};
% \node[anchor=west,font=\small] at (3.85,1.2) {\ref{sec:decision-vc-dimension-counter}};
% \node[anchor=west,font=\small] at (3.85,0.4) {\ref{sec:consistent-pac-learner-inf-vc}};
% \node[anchor=west,font=\small] at (3.85,-0.4) {\ref{sec:consistent-pac-learner-no-strong-compression}};
% \node[anchor=west,font=\small] at (3.85,-1.2) {\ref{sec:pac-learner-no-compression}};
% \end{tikzpicture}
% \caption{The different PAC sufficient conditions and PAC necessary conditions, considered in this work.
% Our results show that some regions are non-empty.
% The question of whether the hatched region is empty or not remains open.}
% \label{fig:results}
% \end{figure}

\begin{table}
\centering
\begin{tabular}{@{}ccccc@{}}
\toprule
& \begin{tabular}{@{}c@{}} Infinite \\ VC dim.
\end{tabular} & \begin{tabular}{@{}c@{}} No Compr.\\Scheme \end{tabular} & \begin{tabular}{@{}c@{}} No Compr.\\Map \end{tabular} & \begin{tabular}{@{}c@{}} Finite \\ dVC dim.
\end{tabular} \\
\midrule
PAC & Sec.~\ref{sec:consistent-pac-learner-inf-vc} & Sec.~\ref{sec:pac-learner-no-compression} & Sec.~\ref{sec:pac-learner-no-compression} & \rotatebox[origin=c]{45}{$\Rightarrow$} \\
Stable PAC & Sec.~\ref{sec:consistent-pac-learner-inf-vc} & Empty? & Sec.~\ref{sec:consistent-pac-learner-no-strong-compression} & \rotatebox[origin=c]{45}{$\Rightarrow$} \\
Not PAC & \rotatebox[origin=c]{45}{$\Rightarrow$}~\cite{alamo2009randomized} & \rotatebox[origin=c]{45}{$\Rightarrow$}~\cite{littlestone1986relating} & \rotatebox[origin=c]{45}{$\Rightarrow$}~\cite{littlestone1986relating} & Sec.~\ref{sec:decision-vc-dimension-counter} \\
\bottomrule
\end{tabular}
\caption{Summary of our results and counterexamples (e.g., Sec.~\ref{sec:consistent-pac-learner-inf-vc} contains an algorithm that is PAC, and stable PAC, and has infinite VC dimension).
The symbol ``\rotatebox[origin=c]{45}{$\Rightarrow$}'' means that the condition of the row implies the condition of the column.}
\label{tab:results}
\vskip-5pt
\end{table}

\subsection{Connections with the Literature}

To the best of our knowledge, this is the first work studying the necessity of PAC conditions for scenario decision algorithms.
% The PAC bound in~\cite{calafiore2006thescenario} is shown to be tight for the class of \emph{fully-supported} convex scenario programs (\cite[Definition~2.5]{calafiore2010random}), thereby providing a form of PAC necessary condition for this class of problems, since it shows that when the dimension of the decision space tends to $\infty$, these problems are not PAC.
The algorithms in~\cite[Theorem~2.1]{de2004onconstraint} and~\cite[Example~23.1]{elyaniv2015ontheversion} are shown to have no compression scheme and be PAC-like (through VC theory); however, the notion of PAC used there is different since the algorithm returns a \emph{set} of decisions, and PAC is defined with respect to the largest risk of a decision in the returned set; see also, e.g.,~\cite[Def.~1]{grammatico2016ascenario}.
The algorithm in~\cite[Appendix~A]{grammatico2016ascenario} is shown to have an unbounded number of support constraints.
With a bit of work, it can be shown to be a valid counterexample for Section~\ref{sec:consistent-pac-learner-no-strong-compression}.
More broadly, we believe that other examples in the literature could potentially serve as counterexamples in Section~\ref{sec:necessity-pac-sufficient}.
However, to the best of our knowledge, these examples have not been explicitly demonstrated to meet the required conditions.
In contrast, our counterexamples are rigorously proven to do so.
Furthermore, they have been carefully designed to be as simple as possible, serving two key purposes: (i) facilitating ease of analysis, and (ii) acting as prototypical examples of problem classes that satisfy the requirements; thereby guiding algorithm designers to identify appropriate PAC sufficient conditions for their problem.
For instance, Section~\ref{sec:consistent-pac-learner-no-strong-compression} exemplifies a nonconvex optimization problem with finite VC dimension.

\paragraph*{Notation}

For $n\in\Ne$, we let $[n]=\{1,\ldots,n\}$.
For a set $A$, we let $A^*=\bigcup_{m=0}^\infty A^m$.

%%%%%%%%%%%%%%%%%%%%%%%%%%%%%%%%%%%%%%%%%%%%%%%%%%%%%%%%%%%%%%%%%%%%%%%%%%%%%%%%
\section{Risk-Aware Scenario Decision Making}
\label{sec:risk-aware-decision-making}

We introduce the problem of risk-aware scenario decision making.%
\footnote{The definitions and concepts presented in this section are classical, and can be found, e.g., in~\cite{alamo2009randomized,margellos2015ontheconnection,garatti2022risk,campi2023compression}.}
We assume given a set $X$ of \emph{decisions} and a set $Z\subseteq2^X$ of \emph{constraints} on $X$.
Given a decision $x\in X$ and a constraint $z\in Z$, we say that $x$ \emph{satisfies} $z$ if $x\in z$; otherwise, we say that $x$ \emph{violates} $z$.

\begin{example}\label{exa:path}
In the following optimization problem:
\[
\min_{x\in\Re^2}\; \lVert x\rVert \quad \text{s.t.} \quad a^\top (x - c) \leq 1 \;\; \forall\, a\in\Re^2,\:\lVert a\rVert\leq1,
\]
where $c\in\Re^2$ is fixed, the decision space is $X=\Re^2$, and the constraint space is $Z=\{C(a)\subseteq X : a\in\Re^2,\:\lVert a\rVert\leq1\}$, where $C(a)=\{x\in\Re^2 : a^\top (x - c)\leq 1\}$.
\end{example}

Finding a decision that satisfies all constraints $z$ in $Z$ is often intractable if $Z$ is large or unknown.
An approach to circumvent this---in the case where $Z$ can be sampled---is to sample $N$ constraints $z_1,\ldots,z_N$ from $Z$, and find a decision $x$ that satisfies the sampled constraints.
The sample-based problem is called the \emph{scenario problem}, and the associated decision the \emph{scenario decision}.
Since the scenario problem is a relaxation of the original problem (aka.~the \emph{robust problem}~\cite{esfahani2015performance}), one can generally not hope that the scenario decision satisfies \emph{all} the constraints in $Z$.
However, under some conditions (the study of necessary and sufficient such conditions is the object of this paper), one can ensure with high probability that the scenario decision satisfies all constraints in $Z$ except possibly those in a subset of small measure.
We formalize this below:

\begin{definition}
Given a probability measure $\Prob$ on $Z$, the \emph{violation probability} (aka.~\emph{risk}) of a decision $x\in X$ w.r.t.~$\Prob$, denoted by $V_\Prob(x)$, is the probability that $x$ violates a randomly chosen constraint $z\in Z$, i.e., $V_\Prob(x) \triangleq \Prob[\{z\in Z : x\notin z\}]$.
\end{definition}

As mentioned above, scenario decision making consists in sampling $N$ constraints from $Z$, and finding a decision that satisfies the sampled constraints.
For instance, one can define the scenario decision as the decision in $X$ that minimizes some cost function while satisfying the sampled constraints: this is the setting of \emph{scenario optimization}~\cite{calafiore2010random,campi2023compression}.
Hence, there is a mapping from tuples of constraints from $Z$ (the samples) to decisions in $X$ (the scenario decision):

\begin{definition}
A \emph{scenario decision algorithm} is a function that given a tuple of constraints $(z_1,\ldots,z_N)\in Z^*$ returns a decision $x\in X$.
Hence, it is a function $\Alg:Z^*\to X$.
\end{definition}

\begin{remark}
We consider \emph{tuples} of constraints, instead of \emph{sets} of constraints, because in general the order and multiplicity of the sampled constraints may influence the scenario decision.
\end{remark}

\begin{example}\label{exa:path-algo}
Continuing Example~\ref{exa:path}, a scenario decision algorithm $\Alg$ can be defined as follows: given $z_i=C(a_i)$ for $i=1,\ldots,N$, $\Alg(z_1,\ldots,z_N)$ is the optimal solution of
\[
\min_{x\in\Re^2}\; \lVert x\rVert \quad \text{s.t.} \quad a_i^\top (x - c) \leq 1 \;\; \forall\, i\in[N].
\]
This is an instance of scenario optimization.
\end{example}

The ability of a scenario decision algorithm to return a decision that has a low risk, provided enough constraints are sampled, is called the \emph{PAC (Probably Approximately Correct) property}.
More precisely, for any PAC algorithm, one can ensure with a predefined probability (called \emph{confidence}) that if enough constraints are sampled, the risk of the decision is below a predefined \emph{tolerance}:

\begin{definition}\label{def:pac-bound}
Consider a scenario decision algorithm $\Alg$.
We say that $\Alg$ is \emph{PAC} if for any tolerance $\epsilon\in(0,1)$ and any confidence $1-\beta\in(0,1)$, there is a sample size $N_\circ\in\Ne$ such that for any probability measure $\Prob$ on $Z$ and any $N\geq N_\circ$, it holds with probability $1-\beta$ that if one samples $N$ constraints $(z_1,\ldots,z_N)$ i.i.d.~according to $\Prob$, then the decision returned by $\Alg$ has risk below $\epsilon$, i.e.,
\[
\Prob^N\!\left[ \left\{ \vz \in Z^N : V_\Prob(\Alg(\vz))>\epsilon \right\} \right] \leq \beta,
\]
where $\vz$ is a shorthand notation for $(z_1,\ldots,z_N)$.
\end{definition}

\begin{remark}\label{rem:set-valued-alg}
In this work, we focus on scenario algorithms that return a single decision, as opposed to scenario algorithms that return a set of decisions, such as those in~\cite{de2004onconstraint,elyaniv2015ontheversion}.
There, the violation probability of a set of decisions is defined as the largest violation probability among all decisions in the set.
This framework is different from ours, and typically precludes the use of compression-based PAC results (defined in the next section) since set-valued scenario decision algorithms rarely admit a compression scheme.
See also Remark~\ref{rem:set-valued-alg-comp}.
\end{remark}

In this work, we focus on scenario decision algorithms that are consistent, i.e., that return a decision that satisfies all the sampled constraints:

\begin{definition}\label{def:consistent}
A scenario decision algorithm $\Alg$ is \emph{consistent} if for all $(z_1,\ldots,z_N)\in Z^*$, $\Alg(z_1,\ldots,z_N)\in\bigcap_{i=1}^N z_i$.
\end{definition}

\begin{remark}
A large class of scenario decision algorithms considered in the literature are consistent~\cite{campi2008theexact,calafiore2010random,campi2018ageneral,campi2018waitandjudge,garatti2021therisk,garatti2022risk,campi2023compression}.
See also Example~\ref{exa:path-algo}.
Scenario decision algorithms that are not consistent include those in~\cite{alamo2009randomized,campi2011asamplinganddiscarding,margellos2015ontheconnection,romao2021tight,romao2023ontheexact}.
Nevertheless, since one of the main goals of this paper is to show that well-known PAC sufficient conditions (see below) are not necessary, there is no loss of generality in showing it under the stronger assumption of consistency.
\end{remark}

%%%%%%%%%%%%%%%%%%%%%%%%%%%%%%%%%%%%%%%%%%%%%%%%%%%%%%%%%%%%%%%%%%%%%%%%%%%%%%%%
\subsection{PAC Sufficient Conditions}

Sufficient conditions for being PAC, available in the literature, can be grouped into two categories: complexity-based conditions and compression-based conditions.
We present one central condition for each category:

\subsubsection{Complexity-Based Condition}

This condition, introduced in~\cite{alamo2009randomized}, relies on the connection between scenario decision making and binary classification learning.
The idea is that each decision $x\in X$ induces a classifier of the constraints $z$ in $Z$ based on whether $x\in z$, or not.
One can then define the \emph{VC dimension} of the set of classifiers induced by $X$ (see Definition~\ref{def:vc}).
It is well known that, in the context of binary classification learning, any set of classifiers with finite VC dimension enjoys PAC properties~\cite[Theorem~6.7]{shalevshwartz2014understanding}.
By leveraging this result,~\cite{alamo2009randomized} obtained PAC sufficient conditions for scenario decision algorithms, as explained below.

For each $x\in X$, we let $\calS(x)=\{z\in Z : x\in z\}$ be the set of constraints in $Z$ that $x$ satisfies.
So, $x$ classifies $Z$ into two classes: $\calS(x)$ and $Z\setminus \calS(x)$.
The \emph{range} of $\Alg$, denoted by $\Rg(\Alg)$, is the set of all classifiers that it can return:
\[
\Rg(\Alg)=\{\calS(\Alg(\vz)) : \vz\in Z^*\}.
\]
We next recall the definition of the VC dimension of a set of classifiers, first introduced by~\cite{vapnik1971ontheuniform}:

\begin{definition}\label{def:vc}
Let $C\subseteq2^Z$ be a set of classifiers.
A subset $Z'\subseteq Z$ is \emph{shattered} by a $C$ if for every classifier $T\subseteq Z'$, there is $S\in C$ such that $T=S\cap Z'$.
The \emph{VC dimension} of $C$ is the supremum of all integers $k$ for which there is a subset $Z'\subseteq Z$ of cardinality $k$ that is shattered by $C$.
\end{definition}

\begin{theorem}[{\cite[Theorem~3]{alamo2009randomized}}]\label{thm:vc-pac-decision}
Consider a consistent scenario decision algorithm $\Alg$.
Assume that $\Rg(\Alg)$ has finite VC dimension.
Then, $\Alg$ is PAC.
\end{theorem}

% \begin{remark}
% Let us mention that compression-based PAC bounds are in general tighter than complexity-based PAC bounds; see, e.g.,~\cite{grammatico2016ascenario,rocchetta2024asurvey}.
% On the other hand, complexity-based bounds can be used to bound the violation probability of \emph{set-valued} scenario decision algorithms (see Remark~\ref{rem:set-valued-alg}).
% Note that the question of whether finiteness of the VC dimension is a PAC sufficient \emph{and} necessary condition for consistent set-valued scenario decision algorithms is answered positively by the theory of binary classification learning~\cite[Theorem~6.7]{shalevshwartz2014understanding}.
% Therefore, in this work, we focus on single-valued scenario decision algorithms.
% \end{remark}

\subsubsection{Compression-Based Condition}

Several compression-based PAC results have been proposed in the literature~\cite{margellos2015ontheconnection,garatti2021therisk,campi2023compression}.
These results rely on the notion of compression reminded below, which may slightly vary across works.
We first introduce the notion of compression used in~\cite{garatti2021therisk,campi2023compression}.
The idea is that a scenario decision algorithm $\Alg$ has compression size $d$ if for any tuple of constraints $(z_1,\ldots,z_N)\in Z^N$, one can extract a subtuple of length at most $d$, from which $\Alg$ gives the same decision as from $(z_1,\ldots,z_N)$.

\begin{definition}\label{def:compression-map}
Consider a scenario decision algorithm $\Alg$.
A \emph{compression map} of capacity $d$ for $\Alg$ is a function $\kappa:Z^*\to Z^{\leq d}$ satisfying that for every $(z_1,\ldots,z_N)\in Z^*$, (i) $\kappa(z_1\ldots,z_N)=(z_{i_1},\ldots,z_{i_r})$ for some integers $1\leq i_1<\ldots<i_r\leq N$, and (ii) $\Alg(\kappa(z_1,\ldots,z_N))=\Alg(z_1,\ldots,z_N)$.
\end{definition}

A slightly more general notion of compression, called here \emph{compression scheme}, was introduced in~\cite{littlestone1986relating} (see also~\cite[\S30]{shalevshwartz2014understanding}) and used in~\cite{margellos2015ontheconnection} for scenario decision making:\footnote{The definition used in~\cite{littlestone1986relating,moran2016sample} is actually slightly more general than Definition~\ref{def:compression-scheme}, because---on top of the indices $\{i_1,\ldots,i_r\}$---an information $q$ belonging to a \emph{finite} information set $Q$ is passed by the compression map to the reconstruction map.
For the sake of simplicity and because our results directly extend to this extended framework, we focus on the slightly more restricted, but simpler, framework.}

\begin{definition}\label{def:compression-scheme}
Consider a scenario decision algorithm $\Alg$.
A \emph{compression scheme} of capacity $d$ for $\Alg$ is a pair $(\kappa,\rho)$, where $\kappa:Z^*\to Z^{\leq d}$ is called the \emph{compression map} and $\rho:Z^{\leq d}\to X$ the \emph{reconstruction map}, satisfying that for every $(z_1,\ldots,z_N)\in Z^*$, (i) $\kappa(z_1\ldots,z_N)=(z_{i_1},\ldots,z_{i_r})$ for some integers $1\leq i_1<\ldots<i_r\leq N$, and (ii) $\rho(\kappa(z_1,\ldots,z_N))=\Alg(z_1,\ldots,z_N)$.
\end{definition}

Clearly, if $\Alg$ admits a compression map of capacity $d$, then it admits a compression scheme of capacity $d$.

\begin{theorem}[{\cite[Theorem~30.2]{shalevshwartz2014understanding}}]\label{thm:compression-pac-decision}
Consider a consistent scenario decision algorithm $\Alg$.
Assume that $\Alg$ admits a compression scheme or map.
Then, $\Alg$ is PAC.
\end{theorem}

%%%%%%%%%%%%%%%%%%%%%%%%%%%%%%%%%%%%%%%%%%%%%%%%%%%%%%%%%%%%%%%%%%%%%%%%%%%%%%%%
\subsection{Stable Scenario Decision Algorithms}

Stability is an important property of many practical scenario decision algorithms, such as scenario optimization~\cite{campi2008theexact,alamo2009randomized,calafiore2010random,campi2011asamplinganddiscarding,margellos2015ontheconnection,campi2018ageneral,garatti2022risk}; see also~\cite{hanneke2016theoptimal,garatti2021therisk,campi2023compression,rocchetta2024asurvey} and Example~\ref{exa:path-algo}.
This property captures the fact that if the returned decision satisfies some unsampled constraint, then adding this constraint to the set of sampled constraints does not change the decision:

\begin{definition}\label{def:stable}
A consistent scenario decision algorithm $\Alg$ is \emph{stable} if for all $(z_1,\ldots,z_{N+1})\in Z^{\geq1}$, $\Alg(z_1,\ldots,z_N)\in z_{N+1}$ implies that $\Alg(z_1,\ldots,z_{N+1})=\Alg(z_1,\ldots,z_N)$.
\end{definition}

In the next section, we study the necessity of PAC sufficient conditions, for both stable and non-stable scenario decision algorithms.

%%%%%%%%%%%%%%%%%%%%%%%%%%%%%%%%%%%%%%%%%%%%%%%%%%%%%%%%%%%%%%%%%%%%%%%%%%%%%%%%%%%%%%%%%%%%%%%%%%%%%%%%%%%%%%%%%%%%%%%%
\section{The VC- and Compression-Based Sufficient Conditions Are Not Necessary}
\label{sec:necessity-pac-sufficient}

The main topic of this paper is to address the key question of whether the sufficient conditions in Theorems~\ref{thm:vc-pac-decision} and~\ref{thm:compression-pac-decision} are also necessary.
It turns out that the answer is negative for general scenario decision algorithms.
Next, we consider the additional assumption of \emph{stability}.
We show that, even under this additional assumption, the answer is mostly negative.
The following theorem summarizes these results (see also Table~\ref{tab:results}):

\begin{theorem}\label{thm:pac-non-necessary}
(i) There exist consistent stable PAC scenario decision algorithms whose range has infinite VC dimension.
(ii) There exist consistent PAC scenario decision algorithms that do not admit a compression scheme.
(iii) There exist stable consistent PAC scenario decision algorithms that do not admit a compression map.
\end{theorem}

The question of existence of a stable PAC scenario decision algorithm without a compression scheme remains open.

We prove Theorem~\ref{thm:pac-non-necessary} by providing counterexamples in the next subsections.
The counterexamples have been simplified to the maximum for the ease of analysis, but they capture the essence of algorithms used in practical problems.
% Hence, they are \emph{prototypical} examples of classes of algorithms that do not satisfy some sufficient conditions but are still PAC.
The purpose is to (i) allow algorithm users and designers to recognize which PAC sufficient condition may be better suited for the analysis of their algorithm, and (ii) build the foundations for a complete characterization of PAC scenario decision algorithms.

All scenario decision algorithms in the rest of this section are consistent.
Hence, we will drop this adjective when referring to them.

%%%%%%%%%%%%%%%%%%%%%%%%%%%%%%%%%%%%%%%%%%%%%%%%%%%%%%%%%%%%%%%%%%%%%%%%%%%%%%%%
\subsection{Stable PAC Scenario Decision Algorithm with Infinite VC Dimension}
\label{sec:consistent-pac-learner-inf-vc}

To prove (i) in Theorem~\ref{thm:pac-non-necessary}, we provide an example of stable PAC scenario decision algorithm $\Alg:Z^*\to X$ whose range has infinite VC dimension.
The algorithm is a convex scenario program (like Example~\ref{exa:path-algo}).
The construction is inspired by~\cite{grelier2021onthevcdimension} and is illustrated in~\cite[Figure~1]{berger2025paclearnability}.

The decision space $X$ is the Euclidean unit ball in $\Re^2$.
To present the constraint set, we need some preliminary constructions.
Let $D\subseteq X$ be the open arc corresponding to the upper-right quarter of the unit circle, i.e., $D=\{(\cos(\alpha),\sin(\alpha)) : 0<\alpha<\frac\pi2\}$.
Let $U$ be the set of all finite subsets of $\Ne_{>0}$, and let $\tau:U\to D\cup\{(1,0)\}$ be a one-to-one function satisfying $\tau(u)\in D$ if $u\neq\emptyset$ and $\tau(\emptyset)=(1,0)$.\footnote{Such a function is given for instance by $\tau(u)= (\cos(\alpha(u)),\sin(\alpha(u)))$, where $\alpha(u)=\frac\pi2n(u)/(1+n(u))$ and $n(u)=\sum_{i\in u} 2^{i-1}$.}
Let $G=\{(m,i)\in\Ne^2:i\leq m\}$, and $\xi:G\to2^U$ be defined by $\xi(m,i)=\{u\subseteq[m] : i\in u\}$, i.e., $\xi(m,i)$ contains all the subsets of $[m]$ that contain $i$.\footnote{E.g., $\xi(3,2)=\{\{2\},\{1,2\},\{2,3\},\{1,2,3\}\}$.}
Finally, let $\sigma:G\to2^X$ be defined by $\sigma(m,i)=\mathrm{convexhull}(\{(0,1)\}\cup\{\tau(u):u\in\xi(m,i)\})$, i.e., $\sigma(m,i)$ is the convex hull of the point $(0,1)$ and the points $\tau(u)$ with $u\in\xi(m,i)$.
Note that $\sigma$ is one-to-one.
With these preliminaries, we define the constraint set by $Z=\{\sigma(m,i):(m,i)\in G\} \cup \{\Re\times[y,1] : y\in[0,1]\}$.
Finally, we define the algorithm $\Alg:Z^*\to X$ as follows: on input $(z_1,\ldots,z_N)\in Z^*$, let
\[\textstyle
\Alg(z_1,\ldots,z_N)=\argmax_{(x_1,x_2)\in X\,\cap\,\bigcap_{i=1}^N z_i} x_1.
\]
It is easy to see that $\Alg$ is well defined, i.e., the argmax exists and is unique.

\begin{proposition}
The VC dimension of $\Rg(\Alg)$ is infinite.
\end{proposition}

\begin{proof}
Fix $k\in\Ne$ and let $Z'=\{\sigma(k,i):i\in[k]\}$.
We show that $Z'$ is shattered by $\Rg(\Alg)$.
For that, fix $T\subseteq Z'$, and we will show that there is $\vz\in Z^*$ such that $\calS(\Alg(\vz))\cap Z'=T$.
Since $\sigma$ is one-to-one, there is $u\subseteq[k]$ such that $T=\{\sigma(k,i) : i\in u\}$.
By definition of $\sigma$, it holds that for all $i\in[k]$, $\tau(u)\in\sigma(k,i)$ if and only if $\sigma(k,i)\in T$.
Hence, to conclude the proof, we need to show that $\tau(u)=\Alg(z)$ for some $z\in X^*$.
This is the case for $z\coloneqq\Re\times[\tau_2(u)),1]$ where $\tau_2(u)$ is the second component of $\tau(u)$.
Hence, the VC dimension of $\Rg(\Alg)$ is at least $k$.
Since $k$ was arbitrary, this concludes the proof.
\end{proof}

The algorithm $\Alg$ is a convex scenario optimization program in finite dimension (the dimension of $X$ is finite).
Hence, it is stable and PAC by classical results in convex scenario optimization; see, e.g.,~\cite[Theorem~2]{margellos2015ontheconnection}.

%%%%%%%%%%%%%%%%%%%%%%%%%%%%%%%%%%%%%%%%%%%%%%%%%%%%%%%%%%%%%%%%%%%%%%%%%%%%%%%%
\subsection{PAC Scenario Decision Algorithm without Compression Scheme}
\label{sec:pac-learner-no-compression}

To prove (ii) in Theorem~\ref{thm:pac-non-necessary}, we provide an example of PAC scenario decision algorithm $\Alg:Z^*\to X$ that does not admit a compression scheme.\footnote{In fact, most PAC scenario optimization programs with sample-dependent objective function (e.g.,~\cite{lauer2024marginbased}) are expected to satisfy this property.
Here, we prove it rigorously for an example that has been crafted for ease of analysis.}
See~\cite[Figure~1]{berger2025paclearnability} for an illustration.

Let $X=\Ne$.
For each $a\in\Ne$, let $U(a)=X\setminus\{a\}$ be the constraint on $X$ requiring that $x\neq a$.
Let $Z=\{U(a) : a\in\Ne\}$.
Define the algorithm $\Alg:Z^*\to X$ as follows: on input $(z_1,\ldots,z_N)\in Z^*$, where for each $i\in[N]$, $z_i=U(a_i)$, let $\Alg(z_1,\ldots,z_N)=1+\sum_{i=1}^N a_i$.

\begin{proposition}
$\Alg$ does not admit a compression scheme.
\end{proposition}

\begin{proof}
Let $k\in\Ne_{>d}$ and $A=\{2^0,2^1,\ldots,2^{k-1}\}$.
For every $S\subseteq2^\Ne$, define $\sigma(S)=\sum_{s\in S}s$.
Note that for any $S_1\subseteq A$ and $S_2\subseteq A$, $S_1\neq S_2$ implies $\sigma(S_1) \neq \sigma(S_2)$.\footnote{Indeed, $S_i$ is the ``binary'' representation of $\sigma(S_i)$.}
Let $T=\{U(a) : a\in A\}$.
The definition of $\Alg$ implies that $\lvert\{\Alg(\vz) : \vz\in T^{\leq k}\}\rvert\geq\lvert\{\sigma(S) : S\subseteq A\}\rvert=2^k$.

For a proof by contradiction, let us assume that $(\kappa,\rho)$ is a compression scheme of capacity $d$ for $\Alg$.
By definition of a compression scheme, it holds that $\{\Alg(\vz) : \vz\in T^N\}\subseteq\{\rho(\vz) : \vz\in T^{\leq d}\}$.
It follows that $\lvert\{\Alg(\vz) : \vz\in Z^N\}\rvert\leq\sum_{r=0}^d\binom{k}{r}\leq(k+1)^d$.
For $k$ large enough, $(k+1)^d<2^k$.
This is a contradiction, concluding the proof.
\end{proof}

We show that $\Alg$ is PAC by showing that its range has VC dimension $1$:

\begin{proposition}\label{prop:pac-learner-no-compression-vc}
The VC dimension of $\Rg(\Alg)$ is $1$.
\end{proposition}

{\def\proof{\noindent\hspace{2em}{\itshape Proof (sketch): }}
\begin{proof}
For any $x\in X$ and $Z'\subseteq Z$, $\calS(x)\cap Z'$ excludes at most one constraint from $Z'$: namely, $\calS(x)\cap Z'=Z'\setminus\{U(x)\}$.
This precludes the existence of a shattered set of size $>1$.
\end{proof}}

\begin{remark}\label{rem:set-valued-alg-comp}
Examples of \emph{set-valued} scenario decision algorithms without a compression scheme, that enjoy \emph{PAC-like properties}, are available in the literature; see, e.g.,~\cite{de2004onconstraint,elyaniv2015ontheversion}.
However, the definition of PAC is different (cf.~Remark~\ref{rem:set-valued-alg}), so that these examples are not valid examples for our problem.
\end{remark}

%%%%%%%%%%%%%%%%%%%%%%%%%%%%%%%%%%%%%%%%%%%%%%%%%%%%%%%%%%%%%%%%%%%%%%%%%%%%%%%%
\subsection{Stable PAC Scenario Decision Algorithm without Compression Map}
\label{sec:consistent-pac-learner-no-strong-compression}

To prove (iii) in Theorem~\ref{thm:pac-non-necessary}, we provide an example of stable PAC scenario decision algorithm $\Alg:Z^*\to X$ that does not admit a compression map.
The construction is similar to the one in Subsection~\ref{sec:pac-learner-no-compression}; hence, most of the details are omitted.

Let $X=\Ne$.
For each $a\in\Ne$, let $U(a)=X\setminus\{a\}$.
Let $Z=\{U(a) : a\in\Ne\}$.
Define the algorithm $\Alg:Z^*\to X$ as follows: on input $(z_1,\ldots,z_N)\in Z^*$, let $\Alg(z_1,\ldots,z_N)=\min \bigcap_{i=1}^N z_i$.

\begin{proposition}
$\Alg$ does not admit a compression map.
\end{proposition}

{\def\proof{\noindent\hspace{2em}{\itshape Proof (sketch): }}
\begin{proof}
Assume that $\kappa$ is a compression map of capacity $d$ for $\Alg$.
For each $i\in\Ne_{>0}$, let $z_i=U(i-1)$.
Observe that $\Alg(z_1,\ldots,z_{d+1})=d$.
Let $\kappa(z_1,\ldots,z_{d+1})=\{i_1,\ldots,i_r\}$ with $r\leq d$.
Note that $\Alg(z_{i_1},\ldots,z_{i_r})<d$, a contradiction.
\end{proof}}

It is clear that $\Alg$ is stable (as it is a nonconvex optimization program).
A proof similar to the proof of Proposition~\ref{prop:pac-learner-no-compression-vc} shows that $\Rg(\Alg)$ has finite VC dimension.
Hence, $\Alg$ is PAC.

%%%%%%%%%%%%%%%%%%%%%%%%%%%%%%%%%%%%%%%%%%%%%%%%%%%%%%%%%%%%%%%%%%%%%%%%%%%%%%%%
\section{VC-Inspired PAC Necessary Condition for Scenario Decision Algorithms}
\label{sec:decision-vc-dimension}

In this section, we propose a new PAC \emph{necessary} condition for scenario decision algorithms.
This condition is inspired by the VC dimension and the associated \emph{no-free-lunch theorem}~\cite{shalevshwartz2014understanding}.
More precisely, we introduce a novel quantity, that can be seen as the ``VC dimension of a scenario decision algorithm'', which we call \emph{dVC dimension}, and we show that finiteness of this quantity is a PAC necessary condition.
We also show with a counterexample that this condition is not a PAC \emph{sufficient} condition.

\begin{definition}
Consider a scenario decision algorithm $\Alg$.
A subset $Z'\subseteq Z$ is \emph{shattered} by $\Alg$ if for every $\vz\in(Z')^*$, it holds that $\calS(\Alg(\vz))\cap Z'=\set(\vz)$, where $\set(z_1,\ldots,z_N)\coloneqq\{z_1,\ldots,z_N\}$.
The \emph{dVC dimension} of $\Alg$ is the supremum of all integers $k$ for which there is a subset $Z'\subseteq Z$ of cardinality $k$ that is shattered by $\Alg$.
\end{definition}

Finiteness of the dVC dimension is a PAC necessary condition:

\begin{theorem}\label{thm:dVC-necessary}
Consider a scenario decision algorithm $\Alg$.
Assume that $\Alg$ is PAC.
Then, $\Alg$ has finite dVC dimension.
\end{theorem}

\begin{proof}
For a proof by contraposition, we assume that $\Alg$ has infinite dVC dimension, and we show that $\Alg$ is not PAC.
Therefore, fix $\epsilon\in(0,\frac12)$.
We show that for all $N\in\Ne$, there is $\Prob$ such that $\Prob^N\!\left[ \{\vz\in Z^N:V_\Prob(\Alg(\vz))>\epsilon\} \right]=1$.
To show that, fix $N\in\Ne$, and let $Z'\subseteq Z$ be a finite set shattered by $\Alg$ with cardinality $\lvert Z'\rvert\geq2N$.
Let $\Prob$ be the discrete probability measure satisfying $\Prob\left[\{z\}\right]=1/\lvert Z'\rvert$ for all $z\in Z'$.
We show that for all $\vz\in(Z')^N$, $V_\Prob(\Alg(\vz))>\epsilon$.
Indeed, fix $\vz\in(Z')^N$.
It holds that $V_\Prob(\Alg(\vz))\geq\frac12$ since $\calS(\Alg(\vz))\cap Z'=\set(\vz)$ (since $Z'$ is shattered) and $\lvert Z'\setminus\set(\vz)\rvert\geq\lvert Z'\rvert/2$.
Since $N$ was arbitrary, this shows that $\Alg$ is not PAC.
\end{proof}

However, finiteness of the dVC dimension is not a PAC sufficient condition, even for stable scenario decision algorithms.
This is shown in the next subsection.

\begin{remark}
The notion of dVC dimension and Theorem~\ref{thm:dVC-necessary} can be extended straightforwardly to set-valued scenario decision algorithms.
However, note that for a certain class of set-valued algorithms, namely those that return all decisions that satisfy the sampled constraints (e.g.,~\cite{de2004onconstraint,elyaniv2015ontheversion}), finiteness of the VC dimension is already a PAC sufficient and \emph{necessary} condition (consequence of the fundamental theorem of PAC learning~\cite[Theorem~6.7]{shalevshwartz2014understanding}).
\end{remark}

%%%%%%%%%%%%%%%%%%%%%%%%%%%%%%%%%%%%%%%%%%%%%%%%%%%%%%%%%%%%%%%%%%%%%%%%%%%%%%%%
\subsection{Stable non-PAC Scenario Decision Algorithm with Finite dVC Dimension}
\label{sec:decision-vc-dimension-counter}

We provide an example of stable consistent scenario decision algorithm $\Alg:Z^*\to X$ that has finite dVC dimension, but is not PAC.

Let $X=2^{[0,1]}$.
For each $a\in[0,1]$, let $U(a)=\{x\in X:a\in x\}$.
Let $Z=\{U(a) : a\in[0,1]\}$.
Define the algorithm $\Alg:Z^*\to X$ as follows: on input $(z_1,\ldots,z_N)\in Z^*$, where for each $i\in[N]$, $z_i=U(a_i)$, if $0\in\{a_i\}_{i=1}^N$, let $\Alg(z_1,\ldots,z_N)=\{a_i\}_{i=1}^N$; otherwise, let $\Alg(z_1,\ldots,z_N)=(0,1]$.

Clearly, $\Alg$ is consistent.
It is also not difficult to show that $\Alg$ is stable.
We show that $\Alg$ has finite dVC dimension:

\begin{proposition}
$\Alg$ has dVC dimension at most $2$.
\end{proposition}

\begin{proof}
Let $Z'\subseteq Z$ be a finite subset with cardinality at least $3$.
We show that $Z'$ is not shattered by $\Alg$.
Indeed, if $U(\{0\})\notin Z'$, then clearly $Z'$ is not shattered by $\Alg$ since for all $\vz\in(Z')^*$, $\Alg(\vz)=(0,1]$, so that $\calS(\Alg(\vz))\cap Z'=Z'$.
On the other hand, if $U(\{0\})\in Z'$, take $a\in(0,1]$, such that $U(a)\in Z'$, and let $\vz=(U(a))$.
Then, $\Alg(\vz)=(0,1]$, so that $\calS(\Alg(\vz))\cap Z'\neq\set(\vz)$ since $\lvert Z'\rvert>0$.
This shows that $Z'$ is not shattered by $\Alg$, concluding the proof.
\end{proof}

\begin{proposition}
$\Alg$ is not PAC.
\end{proposition}

\begin{proof}
Let $\epsilon=\beta=\frac14$.
Consider the continuous-discrete probability measure $\Prob'$ on $[0,1]$ defined by $\Prob'[\{0\}]=\frac12$, and for all $a\in(0,1]$, $\Prob'[(0,a]]=a/2$.
Consider the associated probability distribution on $Z$ defined by: for all $A\subseteq[0,1]$, $\Prob[\{U(a) : a\in A\}]=\Prob'[A]$.
For all $N\in\Ne_{>0}$, the probability of sampling $\vz\in Z^N$ such that $U(\{0\})\in\set(\vz)$ is at least $\frac12$: $\Prob[\{\vz\in Z^N : U(\{0\})\in\set(\vz)\}]=\Prob'[\{\va\in[0,1]^N : 0\in\set(\va)]=1-(\frac12)^N\geq\frac12$.
Whenever $U(\{0\})\in\set(\vz)$, the risk of $\Alg(\vz)$ is $\frac12$, which is larger than $\epsilon$.
Hence, we conclude that for all $N\in\Ne_{>0}$, $\Prob[\{\vz\in Z^N : V_\Prob(\Alg(\vz))>\epsilon\}]\geq \frac12$.
The latter is larger than $\beta$.
Thus, $\Alg$ is not PAC.
\end{proof}

%%%%%%%%%%%%%%%%%%%%%%%%%%%%%%%%%%%%%%%%%%%%%%%%%%%%%%%%%%%%%%%%%%%%%%%%%%%%%%%%
\section{Example of Application of our Results}

Consider the problem of shortest path planning from a starting location $I$ to a target location $T$.
See~\cite[Figure~2]{berger2025paclearnability} for an illustration.
The path has to avoid a \emph{random} obstacle, which takes the form of a barrier positioned at some unknown random angle from the middle point between $I$ and $T$ (see red lines in~\cite[Figure~2]{berger2025paclearnability}).
There is also a fixed known obstacle below $I$ and $T$, namely the gray area in~\cite[Figure~2]{berger2025paclearnability}.
Positions of the random obstacle can be sampled.
We consider two scenario decision algorithms: $\Alg_1$ that computes the shortest path between $I$ and $T$, while avoiding the fixed and sampled obstacles (see magenta line in~\cite[Figure~2]{berger2025paclearnability}); and $\Alg_2$ that computes the shortest \emph{parabola} between $I$ and $T$, while avoiding the fixed and sampled obstacles (see cyan line in~\cite[Figure~2]{berger2025paclearnability}).

It can be shown that the dVC dimension of $\Alg_1$ is infinite.
Hence, this algorithm is not PAC (Theorem~\ref{thm:dVC-necessary}).
On the other hand, it can be shown that $\Alg_2$ admits a compression map of capacity $1$ (take any obstacle that touches the optimal parabola).
Hence, this algorithm is PAC (Theorem~\ref{thm:compression-pac-decision}).

\addtolength{\textheight}{-2cm}  % This command serves to balance the column lengths

%%%%%%%%%%%%%%%%%%%%%%%%%%%%%%%%%%%%%%%%%%%%%%%%%%%%%%%%%%%%%%%%%%%%%%%%%%%%%%%%
\section{Conclusions}

While PAC learning has been a longstanding major topic in Machine Learning, the development of a similar systematic framework for Optimization and Scenario Decision Making remains elusive.
Various PAC sufficient conditions for scenario decision algorithms have been proposed, but the necessity of these conditions remains an open question.
This work addresses this gap by providing counterexamples showing that these conditions are not necessary.
Inspired by practical algorithms used in real-world applications, these counterexamples have been carefully simplified to isolate their essential features.
In addition, we introduce a novel quantity, the dVC dimension, which serves as an analogue to the VC dimension for scenario decision algorithms.
We prove that the finiteness of this dimension is a PAC necessary condition for scenario decision algorithms.
Beyond its theoretical significance in advancing toward a complete characterization of PAC scenario decision algorithms, this work also offers practical insights for algorithm users and designers.
It enables them to identify which PAC sufficient conditions are better suited for analyzing their algorithms---highlighting that some conditions may not hold while others do---, or to determine that their algorithms are not PAC by leveraging necessary conditions.

Looking ahead, we aim to address several open questions.
First, we will investigate the existence of compression schemes for stable PAC scenario decision algorithms.
Second, we will explore new sufficient and necessary PAC conditions, with the goal of narrowing the gap between the two and moving closer to a comprehensive characterization of PAC scenario decision algorithms.

%%%%%%%%%%%%%%%%%%%%%%%%%%%%%%%%%%%%%%%%%%%%%%%%%%%%%%%%%%%%%%%%%%%%%%%%%%%%%%%%

\bibliographystyle{IEEEtran}
\bibliography{myrefs}

\end{document}